\documentclass{svproc}

\usepackage{url}

\usepackage{amsmath}
\usepackage{amssymb}

\usepackage{amsthm}
\usepackage{mathrsfs}
\usepackage{mathtools}
\mathtoolsset{showonlyrefs}
\usepackage{graphicx}
\graphicspath{{figures/}}
\usepackage[caption=false,font=footnotesize]{subfig}
\usepackage{tikz}
\usepackage{multirow}
\usepackage{url}

\DeclareMathOperator*{\minimize}{minimize~}

\DeclareMathOperator*{\subjto}{subject\,to~}
\newcommand{\R}{\mathbb R}
\newcommand{\di}{\mathrm{d}}
\newcommand{\mc}{\mathcal}
\newcommand{\msc}{\mathscr}
\newcommand{\tr}{^T}
\newcommand{\inv}{^{-1}}

\begin{document}

\mainmatter

\title{A Constrained-Optimization Approach to the\\Execution of Prioritized Stacks of\\Learned Multi-Robot Tasks}

\titlerunning{Multi-Robot Multi-Learned-Tasks}

\author{Gennaro Notomista}

\authorrunning{Gennaro Notomista}

\tocauthor{Gennaro Notomista}

\institute{Department of Electrical and Computer Engineering\\University of Waterloo\\Waterloo, ON, Canada\\\email{gennaro.notomista@uwaterloo.ca}}

\maketitle

\begin{abstract}
	This paper presents a constrained-optimization formulation for the prioritized execution of learned robot tasks. The framework lends itself to the execution of tasks encoded by value functions, such as tasks learned using the reinforcement learning paradigm. The tasks are encoded as constraints of a convex optimization program by using control Lyapunov functions. Moreover, an additional constraint is enforced in order to specify relative priorities between the tasks. The proposed approach is showcased in simulation using a team of mobile robots executing coordinated multi-robot tasks.
\end{abstract}

\section{Introduction}

Learning complex robotic tasks can be challenging for several reasons. The nature of compound tasks, made up of several simpler subtasks, renders it difficult to simultaneously capture and combine all features of the subtasks to be learned. Another limiting factor of the learning process of compound tasks is the computational complexity of machine learning algorithms employed in the learning phase. This can make the training phase prohibitive, especially when the representation of the tasks comprises of a large number of parameters, as it is generally the case when dealing with complex tasks made up of several subtasks, or in the case of high-dimensional state space representations.

For these reasons, when there is an effective way of combining the execution of multiple subtasks, it is useful to break down complex tasks into building blocks that can be independently learned in a more efficient fashion. Besides the reduced computational complexity stemming from the simpler nature of the subtasks to be learned, this approach has the benefit of increasing the modularity of the task execution framework, by allowing for a reuse of the subtasks as building blocks for the execution of different complex tasks. Discussions and analyses of such advantages can be found, for instance, in \cite{schwartz1995finding,ghosh2017divide,teh2017distral,nachum2018data}.

Along these lines, in \cite{kaelbling2020foundation}, \textit{compositionality} and \textit{incrementality} are recognized to be two fundamental features of robot learning algorithms. Compositionality, in the context of learning to execute multiple tasks, is intended as the property of learning strategies to be in a form that allows them to be combined with previous knowledge. Incrementality, guarantees the possibility of adding new knowledge and abilities over time, by, for instance, incorporating new tasks. Several approaches have been proposed, which exhibit these two properties. Nevertheless, challenges still remain regarding tasks prioritization and stability guarantees \cite{qureshi2019composing,sahni2017learning,singh1992transfer,van2019composing,dulac2019challenges}. The possibility of prioritizing tasks together with the stability guarantees allows us to characterize the behavior resulting from the composition of multiple tasks.

In fact, when dealing with redundant robotic systems---i.e. systems which possess more degrees of freedom compared to the minimum number required to execute a given task, as, for example, multi-robot systems---it is often useful to allow for the execution of multiple subtasks in a \textit{prioritized stack}. Task priorities may allow robots to adapt to the different scenarios in which they are employed by exhibiting structurally different behaviors. Therefore, it is desirable that a multi-task execution framework allows for the prioritized execution of multiple tasks.

In this paper, we present a constrained-optimization robot-control framework suitable for the \textit{stable} execution of multiple tasks in a \textit{prioritized} fashion. This approach leverages the reinforcement learning (RL) paradigm in order to get an approximation of the value functions which will be used to encode the tasks as constraints of a convex quadratic program (QP). Owing to its convexity, the latter can be solved in polynomial time \cite{boyd2004convex}, and it is therefore suitable to be employed in a large variety of robotic applications, in online settings, even under real-time constraints. The proposed framework shares the optimization-based nature with the one proposed in \cite{notomista2020set} for redundant robotic manipulators, where, however, it is assumed that a representation for all tasks to be executed is known a priori. As will be discussed later in the paper, this framework indeed combines compositionality and incrementality---i.e. the abilities of combining and adding sub-tasks to build up compound tasks, respectively---with stable and prioritized task execution in a computationally efficient optimization-based algorithm.

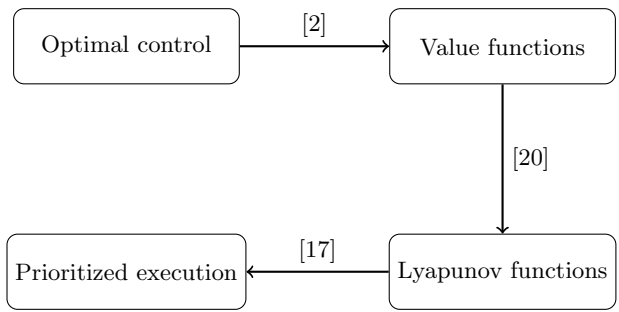
\begin{figure}
\label{fig:blockdiagram}
\centering
\begin{tikzpicture}
	\node (oc) [rectangle, rounded corners, minimum width=3cm, minimum height=1cm, text centered, draw=black] {Optimal control};
	\node (vf) [right of=oc, rectangle, rounded corners, minimum width=3cm, minimum height=1cm, text centered, draw=black, node distance=5cm] {Value functions};
	\node (lf) [below of=vf, rectangle, rounded corners, minimum width=3cm, minimum height=1cm, text centered, draw=black, node distance=3cm] {Lyapunov functions};
	\node (pe) [left of=lf, rectangle, rounded corners, minimum width=3cm, minimum height=1cm, text centered, draw=black, node distance=5cm] {Prioritized execution};
	\draw[->,thick] (oc) -- (vf) node [pos=0.5,above] {\cite{bertsekas2019reinforcement}};
	\draw[->,thick] (vf) -- (lf) node [pos=0.5,right] {\cite{primbs1999nonlinear}};
	\draw[->,thick] (lf) -- (pe) node [pos=0.5,above] {\cite{notomista2019optimal}};
\end{tikzpicture}
\caption{Pictorial representation of the strategy adopted in this paper for the execution of prioritized stacks of learned tasks.}
\end{figure}
Figure~\ref{fig:blockdiagram} pictorially shows the strategy adopted in this paper to allow robots to execute multiple prioritized tasks learned using the RL paradigm. Once a value function is learned using the RL paradigm (using, e.g., the value iteration algorithm \cite{bertsekas2019reinforcement}), this learned value function is used to construct a control Lyapunov function \cite{sontag1983lyapunov} in such a way that a controller synthesized using a min-norm optimization program is equivalent to the optimal policy corresponding to the value function \cite{primbs1999nonlinear}. Then, multiple tasks encoded by constraints in a min-norm controller are combined in a prioritized stack as in \cite{notomista2019optimal}.

To summarize, the contributions of this paper are the following: (i) We present a compositional and incremental framework for the execution of multiple tasks encoded by value functions; (ii) We show how priorities among tasks can be enforced in a constrained-optimization-based formulation; (iii) We frame the prioritized multi-task execution as a convex QP which can be efficiently solved in online settings; (iv) We demonstrate how the proposed framework can be employed to control robot teams to execute coordinated tasks.

\section{Background and Related Work}
\label{sec:relatedwork}

\subsection{Multi-Task Learning, Composition, and Execution}
\label{subsec:mtl}

The prioritized execution framework for learned tasks proposed in this paper can be related to approaches devised for multi-task learning---a machine learning paradigm which aims at leveraging useful information contained in multiple related tasks to help improve the generalization performance of all the tasks \cite{zhang2021survey}. The learning of multiple tasks can happen in parallel (independently) or in sequence for naturally sequential tasks \cite{gupta2021reset,smith2017federated}, and a number of computational frameworks have been proposed to learn multiple tasks (see, e.g., \cite{zhang2021survey,micchelli2004kernels,ruder2017overview}, and references therein). It is worth noticing how, owing to its constrained-optimization nature, the approach proposed in this paper is dual to multi-objective optimization frameworks, such as \cite{sener2018multi,bylard2021composable} or compared to the Riemannian motion policies \cite{ratliff2018riemannian,mukadam2020riemannian,rana2020learning}.

Several works have focused on the composition and hierarchy of deep reinforcement learning policies. The seminal work \cite{todorov2009compositionality} shows compositionality for a specific class of value functions. More general value functions are considered in \cite{haarnoja2017reinforcement}, where, however, there are no guarantees on the policy resulting from the multi-task learning process. Boolean and weighted composition of reward, \mbox{(Q-)value} functions, or policies are considered in \cite{haarnoja2018composable,peng2019mcp,van2019composing}. While these works have shown their effectiveness on complex systems and tasks, our proposed approach differs from them in two main aspects: (i) It separates the task learning from the task composition; (ii) It allows for (possibly time-varying and state-dependent) task prioritization, with task stacks that are enforced at runtime.

\subsection{Constraint-Based Task Execution}
\label{subsect:cbc}

In this paper, we adopt a constrained-optimization approach to the prioritized execution of multiple tasks learned using the RL paradigm. In \cite{notomista2019optimal}, a constraint-based task execution framework is presented for a robotic system with control affine dynamics
\begin{equation}
	\label{eq:dyn}
	\dot x = f_0(x) + f_1(x) u,
\end{equation}
where $x\in\msc X\subseteq \R^n$ and $u\in\msc U\subseteq \R^m$ denote state and control input, respectively. The $M$ tasks to be executed are encoded by continuously differentiable, positive definite cost functions $V_i\colon\msc X\to\R_+,~i=1,\ldots,M$. With the notation which will be adopted in this paper, the constraint-based task execution framework in \cite{notomista2019optimal} can be expressed as follows:
\begin{equation}
	\label{eq:multitaskexecution}
	\begin{aligned}
		\minimize_{u,\delta} & \|u\|^2 + \kappa \|\delta\|^2\\
		\subjto & L_{f_0} V_i(x) + L_{f_1} V_i(x) u \le -\gamma(V_i(x)) + \delta_i \quad i=1,\ldots,M\\
		& K \delta \ge 0,
	\end{aligned}
\end{equation}
where $L_{f_0} V_i(x)$ and $L_{f_1} V_i(x)$ are the Lie derivatives of $V_i$ along the vector fields $f_0$ and $f_1$, respectively. The components of $\delta = [\delta_1,\ldots,\delta_M]\tr$ are used as slack variables employed to prioritize the different tasks; $\gamma\colon\R\to\R$ is a Lipschitz continuous extended class $\mc K$ function---i.e. a continuous, monotonically increasing function, with $\gamma(0) = 0$---$\kappa>0$ is an optimization parameter, and $K$ is the \textit{prioritization matrix}, known a priori, which enforces relative constraints between components of $\delta$ of the following type: $\delta_i \le l \delta_j$, for $l\ll 1$, which encodes the fact that task $i$ is executed at higher priority than task $j$.

In the following, Section~\ref{subsec:dpconstraints} will be devoted to showing the connection between dynamic programming and optimization-based controllers. In Section~\ref{sec:main}, this connection will allow us to execute tasks learned using the RL paradigm by means of a formulation akin to \eqref{eq:multitaskexecution}.

\subsection{From Dynamic Programming to Constraint-Driven Control}
\label{subsec:dpconstraints}

To illustrate how controllers obtained using dynamic programming can be synthesized as the solution of an optimization program, consider a system with the following discrete-time dynamics:
\begin{equation}
	\label{eq:discretedyn}
	x_{k+1} = f(x_k,u_k).
\end{equation}
These dynamics can be obtained, for instance, by \eqref{eq:dyn}, through a discretization process. In \eqref{eq:discretedyn}, $x_k$ denotes the state, $u_k\in\msc U_k(x_k)$ the input, and the input set $\msc U_k(x_k)$ may depend in general on the time $k$ and the state $x_k$. The value iteration algorithm to solve a deterministic dynamic programming problem with no terminal cost can be stated as follows \cite{bertsekas2019reinforcement}:
\begin{equation}
	\label{eq:VI}
	J_{k+1}(x_k) = \min_{u_k\in \msc U_k(x_k)} \bigg\{ g_k(x_k,u_k) + J_k(f_k(x_k,u_k)) \bigg\},
\end{equation}
with $J_0(x_0) = 0$, where $x_0$ is the initial state, and $g_k(x_k,u_k)$ is the cost incurred at time $k$. The total cost accumulated along the system trajectory is given by
\begin{equation}
	\label{eq:costDP}
	J(x_0) = \lim_{N\to\infty} \sum_{k=0}^{N-1} \alpha^k g_k(x_k,u_k).
\end{equation}
In this paper, we will consider $\alpha = 1$ and we will assume there exists a cost-free termination state.\footnote{Problems of this class are referred to as shortest path problems in \cite{bertsekas2019reinforcement}}.

By Proposition 4.2.1 in \cite{bertsekas2019reinforcement} the value iteration algorithm \eqref{eq:VI} converges to $J^\star$ satisfying
\begin{equation}
	\label{eq:Jstar}
	J^\star(x) = \min_{u\in\msc U(x)} \bigg\{ g(x,u) + J^\star(f(x,u)) \bigg\}.
\end{equation}

Adopting an approximation scheme in value space, $J^\star$ can be replaced by its approximation $\tilde J^\star$ by solving the following approximate dynamic programming algorithm:
\begin{equation}
	\nonumber
	\tilde J_{k+1}(x_k) = \min_{u_k\in \msc U_k(x_k)} \bigg\{ g_k(x_k,u_k) + \tilde J_k(f_k(x_k,u_k)) \bigg\}.
\end{equation}
In these settings, deep RL algorithms can be leveraged to find parametric approximations, $\tilde J^\star$, of the value function using neural networks. This will be the paradigm considered in this paper in order to approximate value functions encoding the tasks to be executed in a prioritized fashion.

The bridge between dynamic programming and constraint-driven control is optimal control. In fact, the cost in \eqref{eq:costDP} is typically considered in optimal control problems, recalled, in the following, for the continuous time control affine system \eqref{eq:dyn}:
\begin{equation}
	\label{eq:optimalcontrolproblem}
	\begin{aligned}
		\minimize_{u(\cdot)} &\int_0^\infty \left( q(x(t)) + u(t)\tr u(t) \right) \di t\\
		\subjto & \dot x = f_0(x) + f_1(x) u.
	\end{aligned}
\end{equation}
Comparing \eqref{eq:optimalcontrolproblem} with \eqref{eq:costDP}, we recognize that the instantaneous cost $g(x,u)$ in \eqref{eq:costDP} in the context of the optimal control problem \eqref{eq:optimalcontrolproblem} corresponds to $q(x) + u\tr u$, where $q\colon\msc X\to\R$ is a continuously differentiable and positive definite function.

A dynamic programming argument on \eqref{eq:optimalcontrolproblem} leads to the following Hamilton-Jacobi-Bellman equation:
\begin{equation}
	\nonumber
	L_{f_0} J^\star(x) - \frac{1}{4} L_{f_1} J^\star(x) \left(L_{f_1} J^\star(x) \right)\tr + q(x) = 0,
\end{equation}
where $J^*$ is the value function---similar to \eqref{eq:Jstar} for continuous-time problems---representing the minimum cost-to-go from state $x$, defined as
\begin{equation}
	\label{eq:optimalcosttogo}
	J^\star(x) = \min_{u(\cdot)} \int_t^\infty \left( q(x(\tau)) + u(\tau)\tr u(\tau) \right) \di\tau.
\end{equation}
The optimal policy corresponding to the optimal value function \eqref{eq:optimalcosttogo} can be evaluated as follows \cite{bryson2018applied}:
\begin{equation}
	\label{eq:optimalpolicy}
	u^\star = -\frac{1}{2} \left(L_{f_1} J^\star(x)\right)\tr.
\end{equation}

In order to show how the optimal policy $u^\star$ in \eqref{eq:optimalpolicy} can be obtained using an optimization-based formulation, we now recall the concept of control Lyapunov functions.

\begin{definition}[Control Lyapunov function \cite{sontag1983lyapunov}]
	\label{def:clf}
	A continuously differentiable, positive definite function $V\colon\R^n\to\R$ is a control Lyapunov function (CLF) for the system \eqref{eq:dyn} if, for all $x\neq0$
	\begin{equation}
		\label{eq:clfcondition}
		\inf_u \bigg\{ L_{f_0} V(x) + L_{f_1} V(x) u \bigg\} < 0.
	\end{equation}
\end{definition}

To select a control input $u$ which satisfies the inequality \eqref{eq:clfcondition}, a universal expression---known as the Sontag's formula \cite{sontag1989universal}---can be employed. With the aim of encoding the optimal control input $u^\star$ by means of a CLF, we will consider the following modified Sontag's formula originally proposed in \cite{freeman1996control}:
\begin{equation}
	\label{eq:sontagmodified}
	u(x) = \begin{cases}
		-v(x) \left(L_{f_1} V(x) \right)\tr & \mathrm{if}~L_{f_1} V(x) \neq 0\\
		0 & \mathrm{otherwise},\\
	\end{cases}
\end{equation}
where $v(x) = \frac{L_{f_0} V(x) + \sqrt{\left(L_{f_0} V(x)\right)^2 + q(x) L_{f_1} V(x) \left(L_{f_1} V(x) \right)\tr}}{L_{f_1} V(x) \left(L_{f_1} V(x) \right)\tr}$.

As shown in \cite{primbs1999nonlinear}, the modified Sontag's formula \eqref{eq:sontagmodified} is equivalent to the solution of the optimal control problem \eqref{eq:optimalcontrolproblem} if the following relation between the CLF $V$ and the value function $J^\star$ holds:
\begin{equation}
	\label{eq:CLFoptimal}
	\frac{\partial J^\star}{\partial x} = \lambda(x) \frac{\partial V}{\partial x},
\end{equation}
where $\lambda(x) = 2 v(x) \left(L_{f_1} V(x) \right)\tr$.
The relation in \eqref{eq:CLFoptimal} corresponds to the fact that the level sets of the CLF $V$ and those of the value function $J^\star$ are parallel.

The last step towards the constrained-optimization-based approach to generate optimal control policies is to recognize the fact that, owing to its inverse optimality property \cite{freeman1996control}, the modified Sontag's formula \eqref{eq:sontagmodified} can be obtained using the following constrained-optimization formulation, also known as the pointwise min-norm controller:
\begin{equation}
	\label{eq:minnorm}
	\begin{aligned}
		\minimize_{u} & \|u\|^2\\
		\subjto & L_{f_0} V(x) + L_{f_1} V(x) u \le -\sigma(x),
	\end{aligned}
\end{equation}
where $\sigma(x) = \sqrt{\left(L_{f_0} V(x)\right)^2 + q(x) L_{f_1} V(x) \left(L_{f_1} V(x) \right)\tr}$. This formulation shares the same optimization structure with the one introduced in \eqref{eq:multitaskexecution} in Section~\ref{sec:relatedwork}, and in the next section we will provide a formulation which strengthens the connection with approximate dynamic programming.

In Appendix~\ref{app:equivalence}, additional results are reported, which further illustrate the theoretical equivalence discussed in this section, by comparing the optimal controller, the optimization-based controller, and a policy learned using the RL framework for a simple dynamical system.

\section{Prioritized Multi-Task Execution}
\label{sec:main}

When $V = \tilde J^\star$, the min-norm controller solution of \eqref{eq:minnorm} is the optimal policy which would be learned using a deep RL algorithm. This is what allows us to bridge the gap between constraint-driven control and RL and it is the key to execute tasks learned using the RL paradigm in a compositional, incremental, prioritized, and computationally-efficient fashion.

Following the formulation given in \eqref{eq:multitaskexecution}, the multi-task prioritized execution of tasks learned using RL can be implemented executing the control input solution of the following optimization program:
\begin{equation}
	\label{eq:multiRLtaskexecution}
	\begin{aligned}
		\minimize_{u,\delta} & \|u\|^2 + \kappa \|\delta\|^2\\
		\subjto & \frac{1}{\lambda_i(x)} \left(L_{f_0} \tilde J_i^\star(x) + L_{f_1} \tilde J_i^\star(x) u \right) \le -\sigma_i(x) + \delta_i,\quad i=1,\ldots,M\\
		& K \delta \ge 0
	\end{aligned}
\end{equation}
where $\tilde J_1^\star,\ldots,\tilde J_M^\star$ are the approximated value functions encoding the tasks learned using the RL paradigm (e.g. value iteration). In summary, with the RL paradigm, one can get the approximate value functions $\tilde J_1^\star,\ldots,\tilde J_M^\star$; the robotic system is then controlled using the control input solution of \eqref{eq:multiRLtaskexecution} in order to execute these tasks in a prioritized fashion.

\begin{remark}
	The Lie derivatives $L_{f_0} \tilde J_1^\star(x),\ldots,L_{f_0} \tilde J_M^\star(x)$ contain the gradients $\frac{\partial J_1^\star}{\partial x},\ldots,\frac{\partial J_M^\star}{\partial x}$. When $\tilde J_1^\star(x),\ldots,\tilde J_M^\star(x)$ are approximated using neural networks, these gradients can be efficiently computed using back propagation.
\end{remark}

We conclude this section with the following Proposition~\ref{prop:stability}, which ensures the stability of the prioritized execution of multiple tasks encoded through the value functions $\tilde J_i^\star$ by a robotic system modeled by the dynamics \eqref{eq:dyn} and controlled with control input solution of \eqref{eq:multiRLtaskexecution}.

\begin{proposition}[Stability of multiple prioritized learned tasks]
	\label{prop:stability}
	Consider executing a set of $M$ prioritized tasks encoded by approximate value functions $\tilde J^\star_i,~i=1,\ldots,M$, by solving the optimization problem in~\eqref{eq:multiRLtaskexecution}. Assume the following:
	\begin{enumerate}
		\item All constraints in \eqref{eq:multiRLtaskexecution} are active
		\item The robotic system can be modeled by driftless control affine dynamical system, i.e. $f_0(x)=0~\forall x\in\msc X$
		\item The instantaneous cost function $g$ used to learn the tasks is positive for all $x\in\msc X$ and $u\in\msc U$.
	\end{enumerate}
	Then,
	\begin{equation}
		\nonumber
		\begin{bmatrix}
			\tilde J_1^\star(x(t))\\
			\vdots\\
			\tilde J_M^\star(x(t))
		\end{bmatrix}\to \mc N(K),
	\end{equation}
	as $t\to\infty$, where $\mc N(K)$ denotes the null space of the prioritization matrix $K$. That is, the tasks will be executed according to the priorities specified by the prioritization matrix $K$ in \eqref{eq:multitaskexecution}.
\end{proposition}
\begin{proof}
	The Lagrangian associated with the optimization problem \eqref{eq:multiRLtaskexecution} is given by
	$L(u,\delta) = \|u\|^2 + \kappa \|\delta\|^2 + \eta_1\tr \left( \hat f_0(x) + \hat f_1(x) u + \sigma(x) - \delta \right) + \eta_2\tr (-K\delta)$,
	where $\hat f_0(x)\in\R^M$ and $\hat f_1(x)\in\R^{M\times m}$ defined as follows: the $i$-th component of $\hat f_0(x)$ is equal to $\frac{1}{\lambda_i(x)} L_{f_0} \tilde J_i^\star(x)$, while the $i$-th row of $\hat f_i(x)$ is equal to $\frac{1}{\lambda_i(x)} L_{f_1} \tilde J_i^\star(x)$. $\eta_1$ and $\eta_2$ are the Lagrange multipliers corresponding to the task and prioritization constraints, respectively.
	
	From the KKT conditions, we obtain:
	\begin{equation}
		\label{eq:proof:udelta}
		u = -\frac{1}{2}\begin{bmatrix}\hat f_1(x)\tr & 0\end{bmatrix}\eta\qquad
		\delta = \frac{1}{2\kappa}\begin{bmatrix}I & -K\tr\end{bmatrix}\eta,
	\end{equation}
	where $\eta = [\eta_1\tr, \eta_2\tr]\tr$. By resorting to the Lagrange dual problem, and by using assumption 1, we get the following expression for $\eta$:
	\begin{equation}
		\label{eq:proof:eta}
		\eta = 2 \underbrace{\begin{bmatrix}
				\frac{I}{\kappa}+\hat f_1(x) \hat f_1(x)\tr & -K\tr \\
				K & K K\tr
			\end{bmatrix}\inv}_{A_1\inv} \underbrace{\begin{bmatrix}
				\hat f_0(x) + \sigma(x)\\
				0
		\end{bmatrix}}_{b_0},
	\end{equation}
	where $I$ denotes an identity matrix of appropriate size. Substituting \eqref{eq:proof:eta} in \eqref{eq:proof:udelta}, we get $u = -\begin{bmatrix}\hat f_1(x)\tr & 0\end{bmatrix} A_1\inv b_0$ and $\delta = \frac{1}{\kappa}\begin{bmatrix}I & -K\tr\end{bmatrix} A_1\inv b_0$.
	
	To show the claimed stability property, we will proceed by a Lyapunov argument. Let us consider the Lyapunov function candidate
	$V(x) = \frac{1}{2} \tilde J(x)^{\star T} K\tr K \tilde J^\star(x)$,
	where $\tilde J^\star(x) = \begin{bmatrix}
		\tilde J_1^\star(x) & \ldots & \tilde J_M^\star(x)
	\end{bmatrix}\tr$.
	The time derivative of $V$ evaluates to:
	\begin{equation}
		\nonumber
		\begin{aligned}
			\dot V &= \frac{\partial V}{\partial x} \dot x = \tilde J(x)^{\star T} K\tr K \frac{\partial \tilde J^\star}{\partial x} \dot x\\
			&= \tilde J(x)^{\star T} K\tr K \underbrace{\frac{\partial \tilde J^\star}{\partial x} f_1(x)}_{\hat f_{1,\lambda}(x)} u \qquad\text{(by assumption 2)},
		\end{aligned}
	\end{equation}
	where, notice that $\hat f_{1,\lambda}(x) = \Lambda(x) \hat f_{1}(x) \in\R^M$ and $\Lambda(x) = \mathrm{diag}\left( \left[ \lambda_1(x), \ldots, \lambda_M(x) \right] \right)$. By assumption 2, $\lambda_i(x)\ge0$ for all $i$, and therefore $\Lambda(x)\succeq0$, i.e. $\Lambda(x)$ is positive semidefinite. Then, $\dot V = \tilde J(x)^{\star T} K\tr K \Lambda(x) \hat f_{1}(x) u = -\tilde J(x)^{\star T} K\tr K \Lambda(x) \hat A \begin{bmatrix}
		\sigma(x)\\
		0
	\end{bmatrix}$,
	where
	\begin{equation}
		\hat A = \hat f_{1}(x) \begin{bmatrix}\hat f_1(x)\tr & 0\end{bmatrix} \begin{bmatrix}
			\frac{I}{\kappa}+\hat f_1(x) \hat f_1(x)\tr & -K\tr \\
			K & K K\tr
		\end{bmatrix}\inv \succeq0
	\end{equation}
	as in Proposition 3 in \cite{notomista2019optimal}, and we used assumption 2 to simplify the expression of $b_0$.
	
	By assumption 3, it follows that the value functions $\tilde J^\star_i$ are positive definite. Therefore, from the definition of $\sigma$, in a neighborhood of $0\in\R^M$, we can bound $\sigma(x)$---defined by the gradients of $\tilde J^\star$---by the value of $\tilde J^\star$ as $\sigma(x) = \gamma_J (\tilde J^\star(x))$, where $\gamma_J$ is a class $\mc K$ function.
	
	Then, proceeding similarly to Proposition 3 in \cite{notomista2019optimal}, we can bound $\dot V$ as follows:
	$\dot V = -\tilde J(x)^{\star T} K\tr K \Lambda(x) \hat A \gamma_J(\tilde J(x)) \le -\tilde J(x)^{\star T} K\tr K \Lambda(x) \tilde J(x) \le -\bar\lambda V(x)$,
	where $\bar\lambda = \min\{\lambda_1(x), \ldots, \lambda_M(x) \}$. Hence, $K \tilde J^\star(x(t)) \to 0$ as $t\to\infty$, and $\tilde J^\star(x(t)) \to \msc N(K)$ as $t\to\infty$.
\end{proof}

\begin{remark}
	The proof of Proposition~\ref{prop:stability} can be carried out even in case of time-varying and state-dependent prioritization matrix $K$. Under the assumption that $K$ is bounded and continuously differentiable for all $x$ and uniformly in time, the norm and the gradient of $K$ can be bounded in order to obtain an upper bound for $\dot V$.
\end{remark}

\begin{remark}
	Even when the prioritization stack specified through the matrix $K$ in \eqref{eq:multiRLtaskexecution} is not physically realizable---due to the fact that, for instance, the functions encoding the tasks cannot achieve the relative values prescribed by the prioritization matrix---the optimization program will still be feasible. Nevertheless, the tasks will not be executed with the desired priorities and even the execution of high-priority tasks might be degraded.
\end{remark}

\section{Experimental Results}
\label{sec:result}

In this section, the proposed framework for the execution of prioritized stacks of tasks is showcased in simulation using a team of mobile robots. Owing to the multitude of robotic units of which they are comprised, multi-robot systems are often highly redundant with respect to the tasks they have to execute. Therefore, they perfectly lend themselves to the concurrent execution of multiple prioritized tasks.

\subsection{Multi-Robot Tasks}
\label{subsec:multirobot}

\begin{figure*}
	\centering
	\subfloat[][t = 0s]{\label{subfig:multi1}\includegraphics[width=0.245\textwidth]{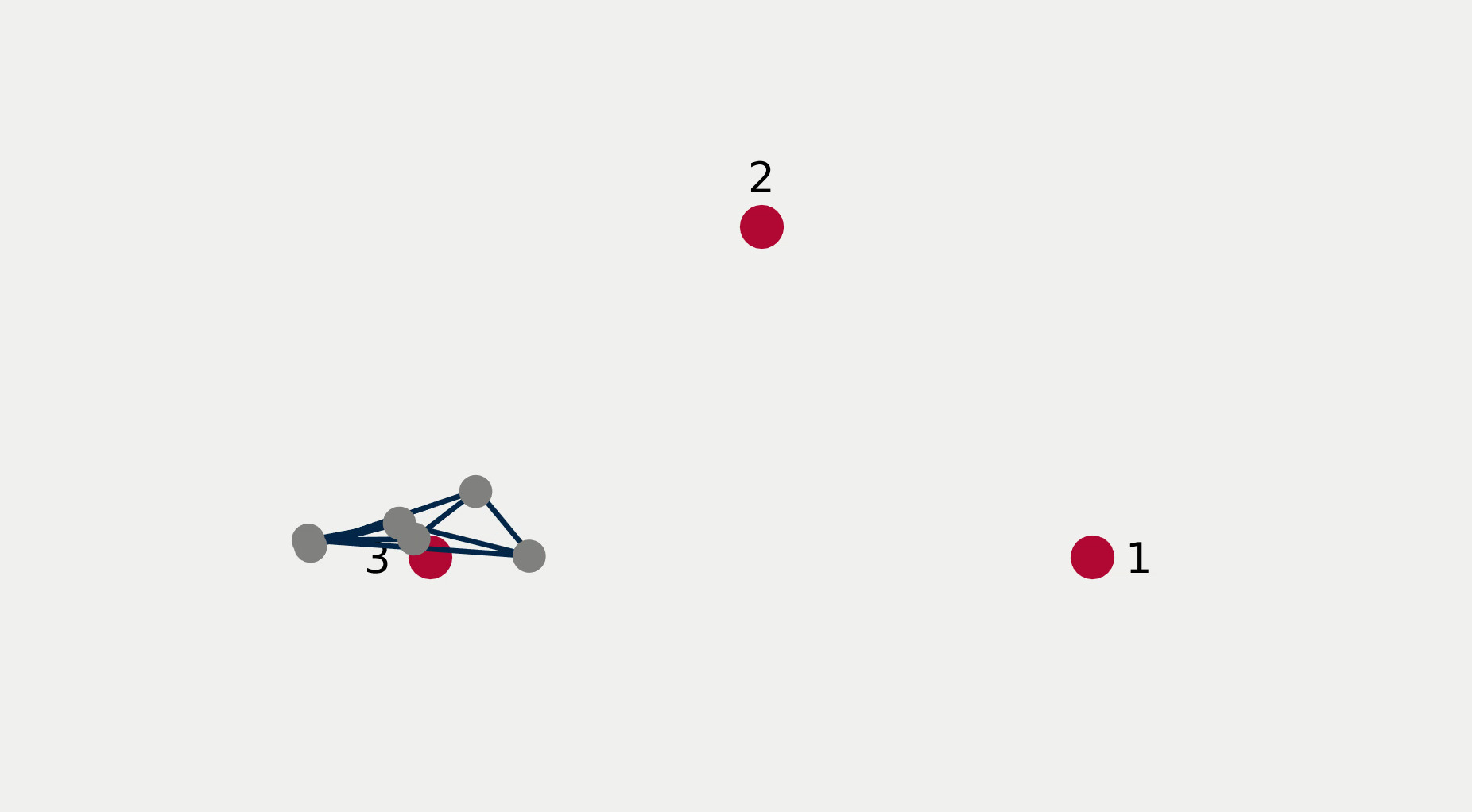}}\hfill
	\subfloat[][t = 5s]{\includegraphics[width=0.245\textwidth]{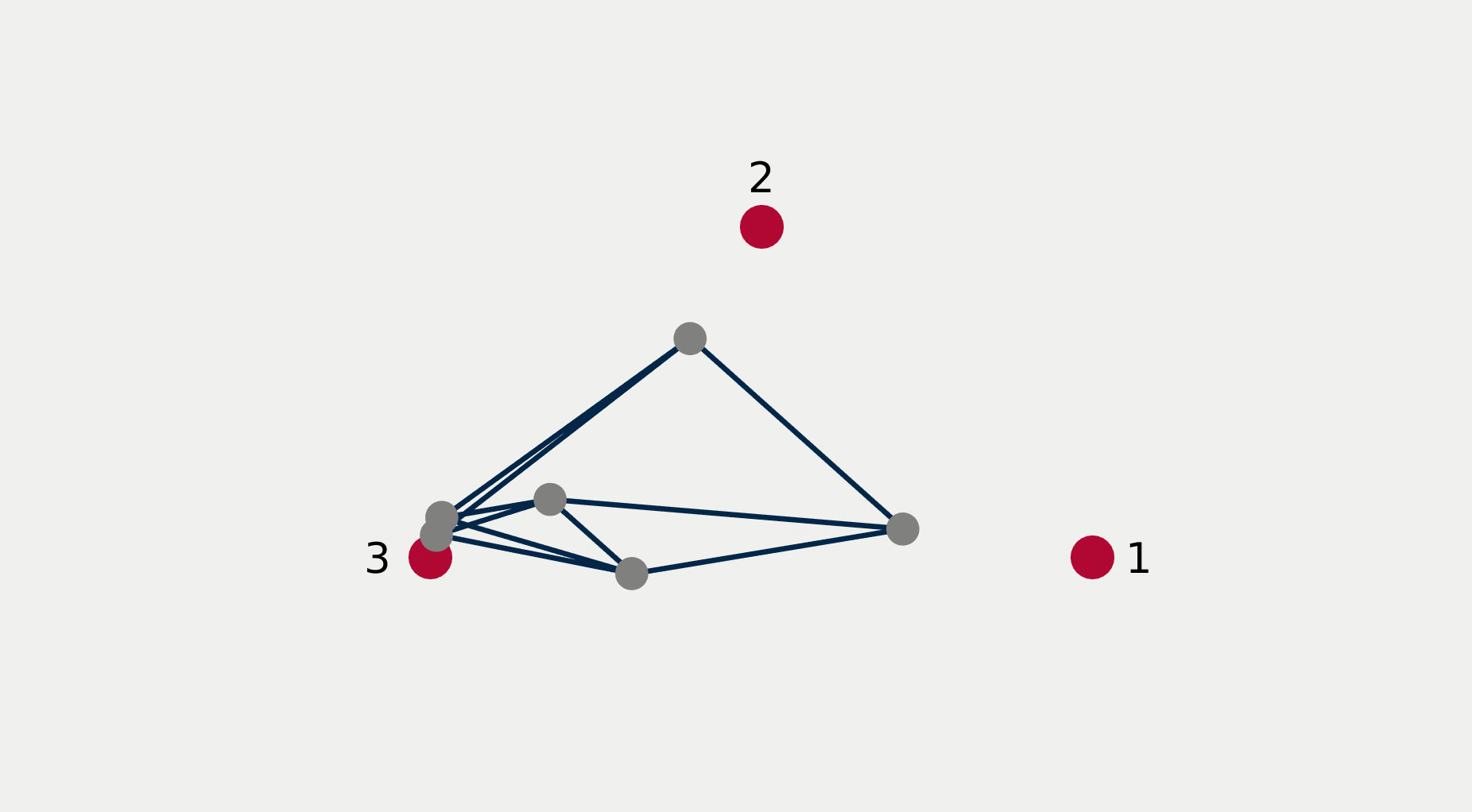}}\hfill
	\subfloat[][t = 10s]{\includegraphics[width=0.245\textwidth]{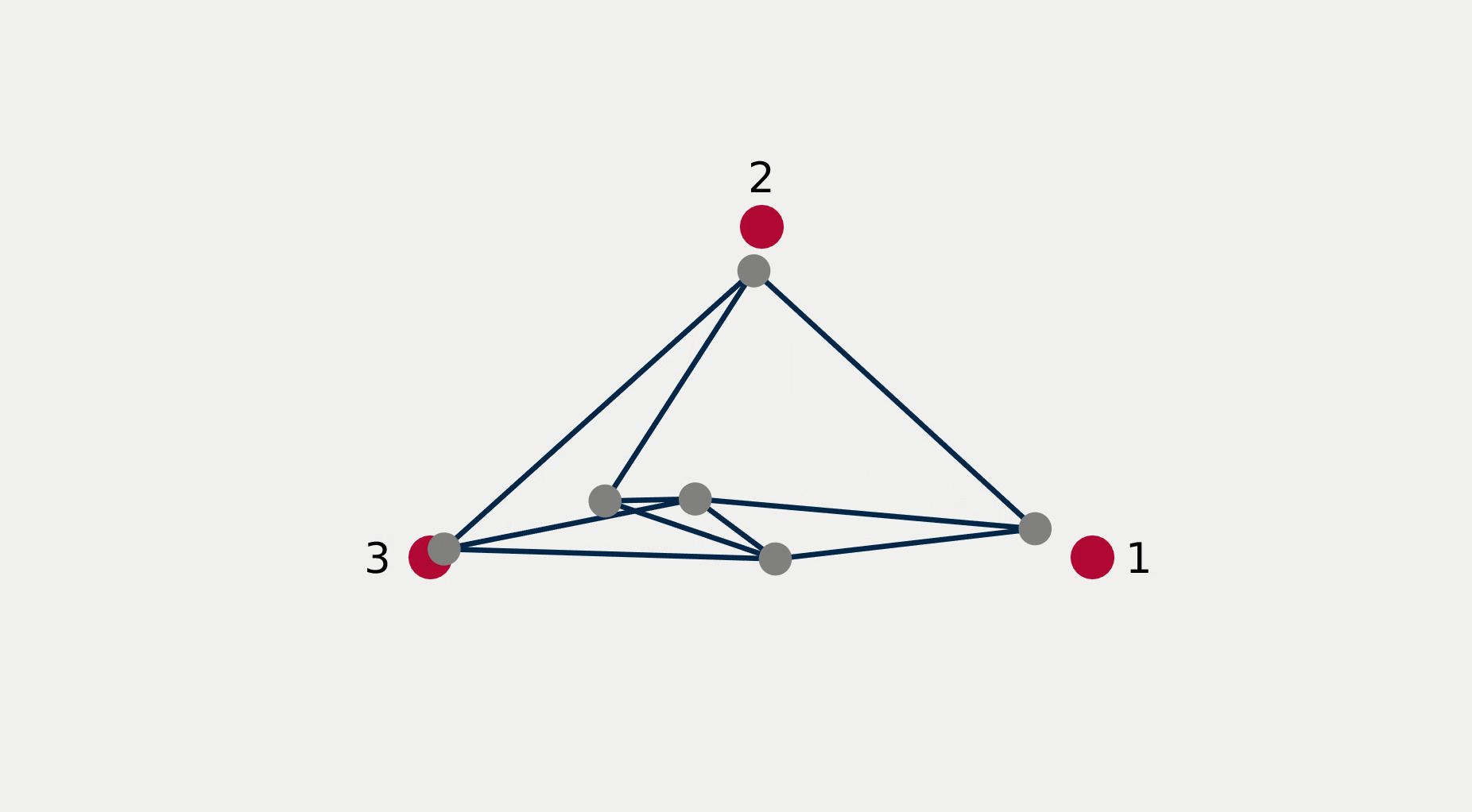}}\hfill
	\subfloat[][t = 15s]{\includegraphics[width=0.245\textwidth]{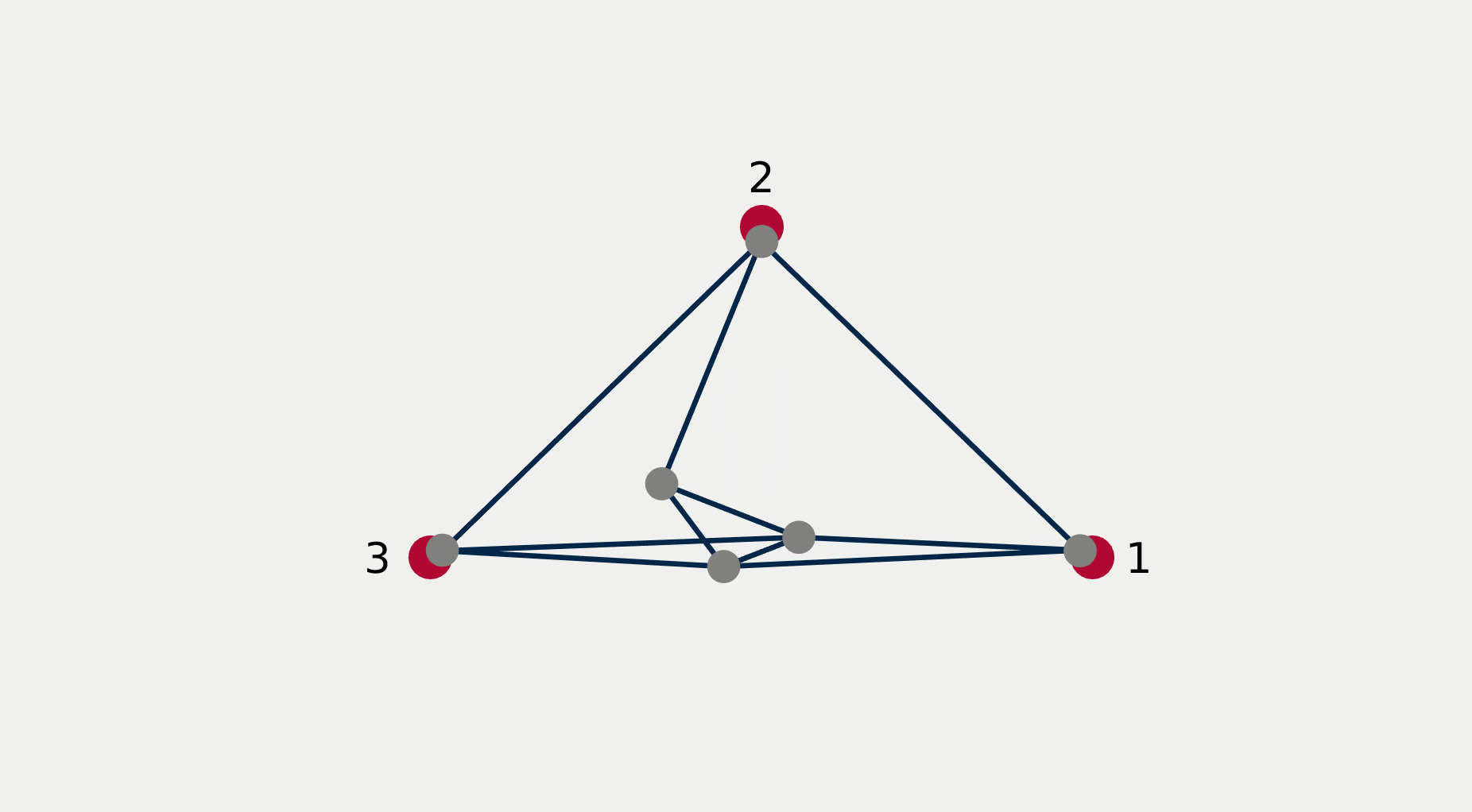}}\\
	\subfloat[][t = 19s]{\includegraphics[width=0.245\textwidth]{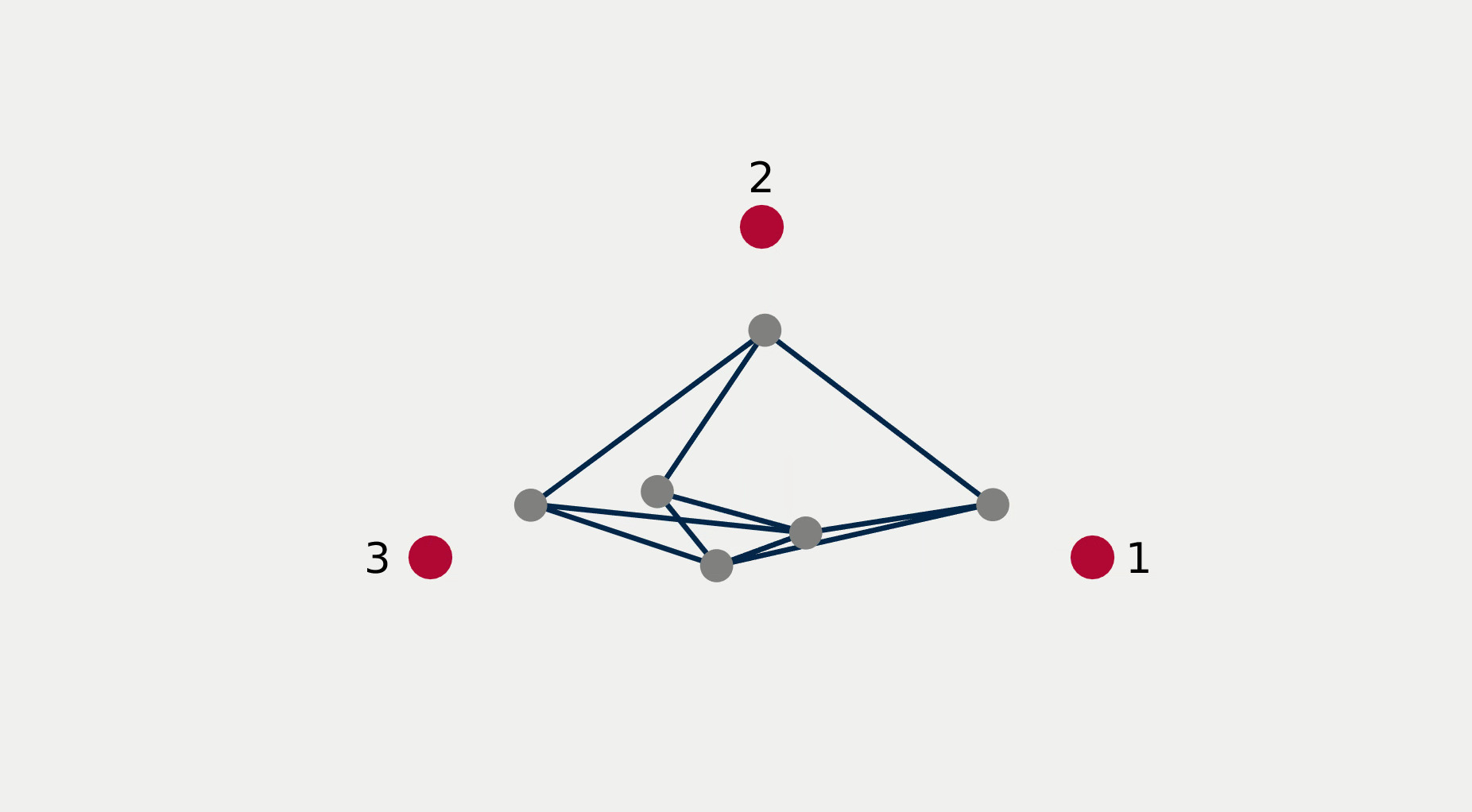}}\hfill
	\subfloat[][t = 23s]{\includegraphics[width=0.245\textwidth]{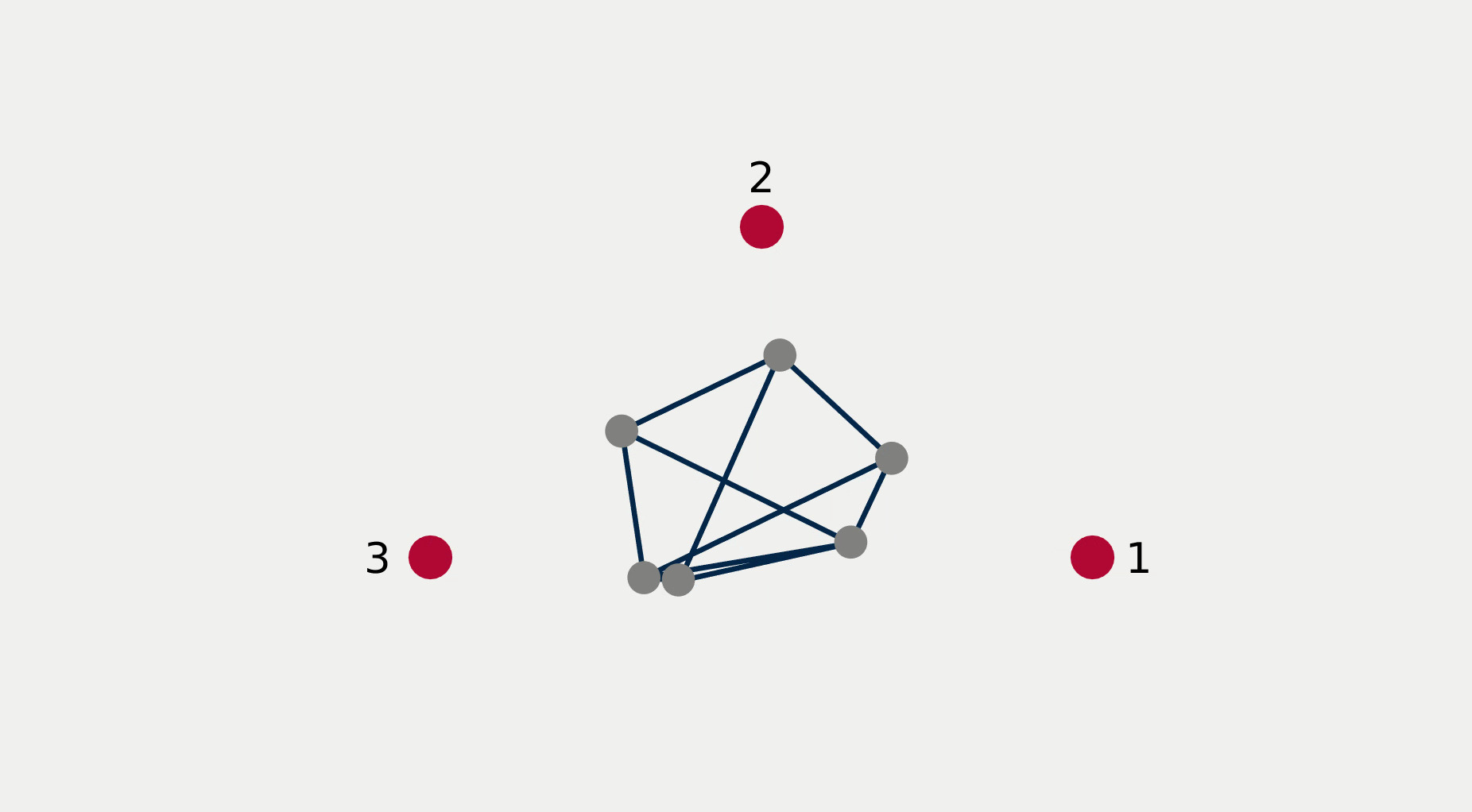}}\hfill
	\subfloat[][t = 27s]{\includegraphics[width=0.245\textwidth]{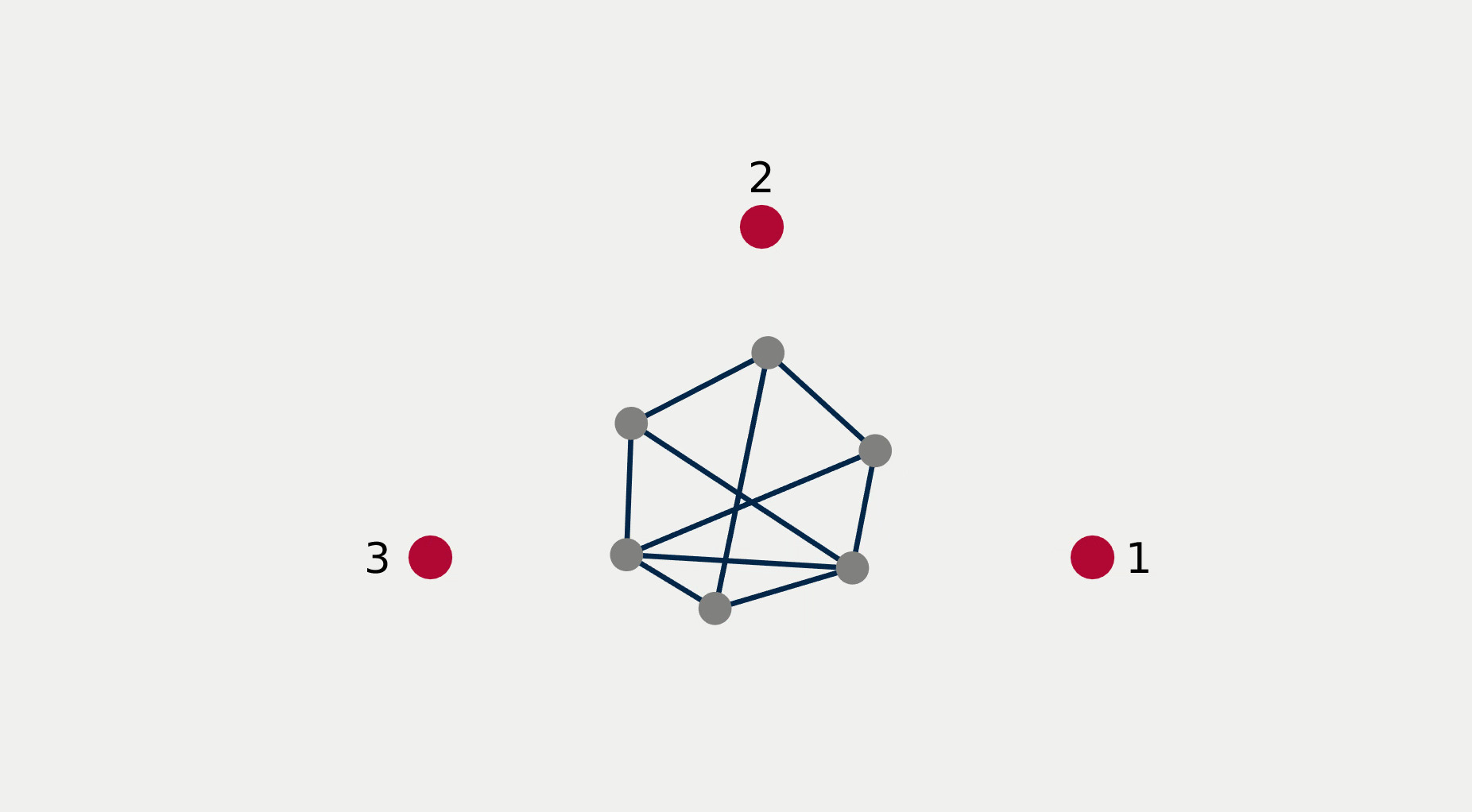}}\hfill
	\subfloat[][t = 30s]{\includegraphics[width=0.245\textwidth]{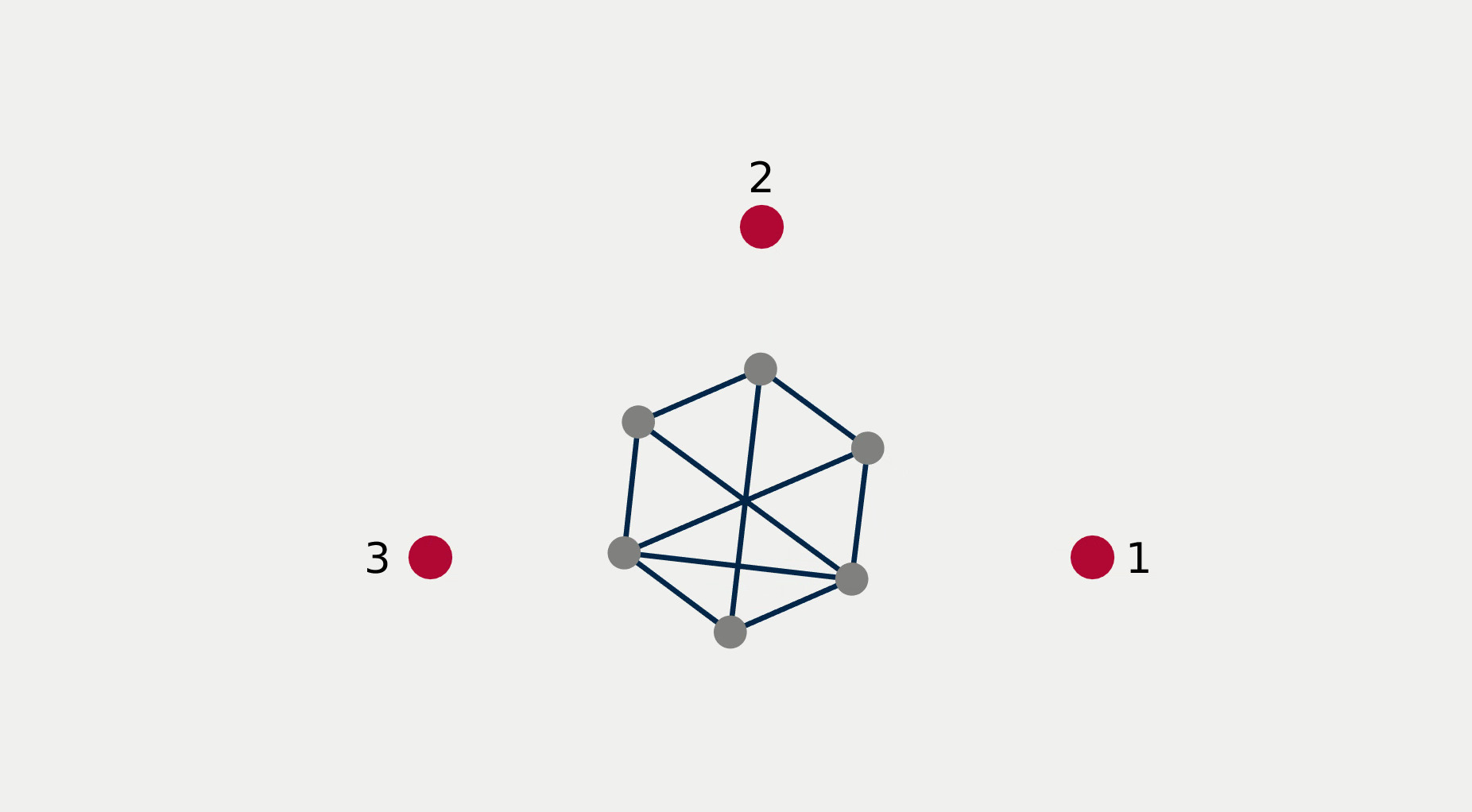}}\\
	\subfloat[][t = 34s]{\includegraphics[width=0.245\textwidth]{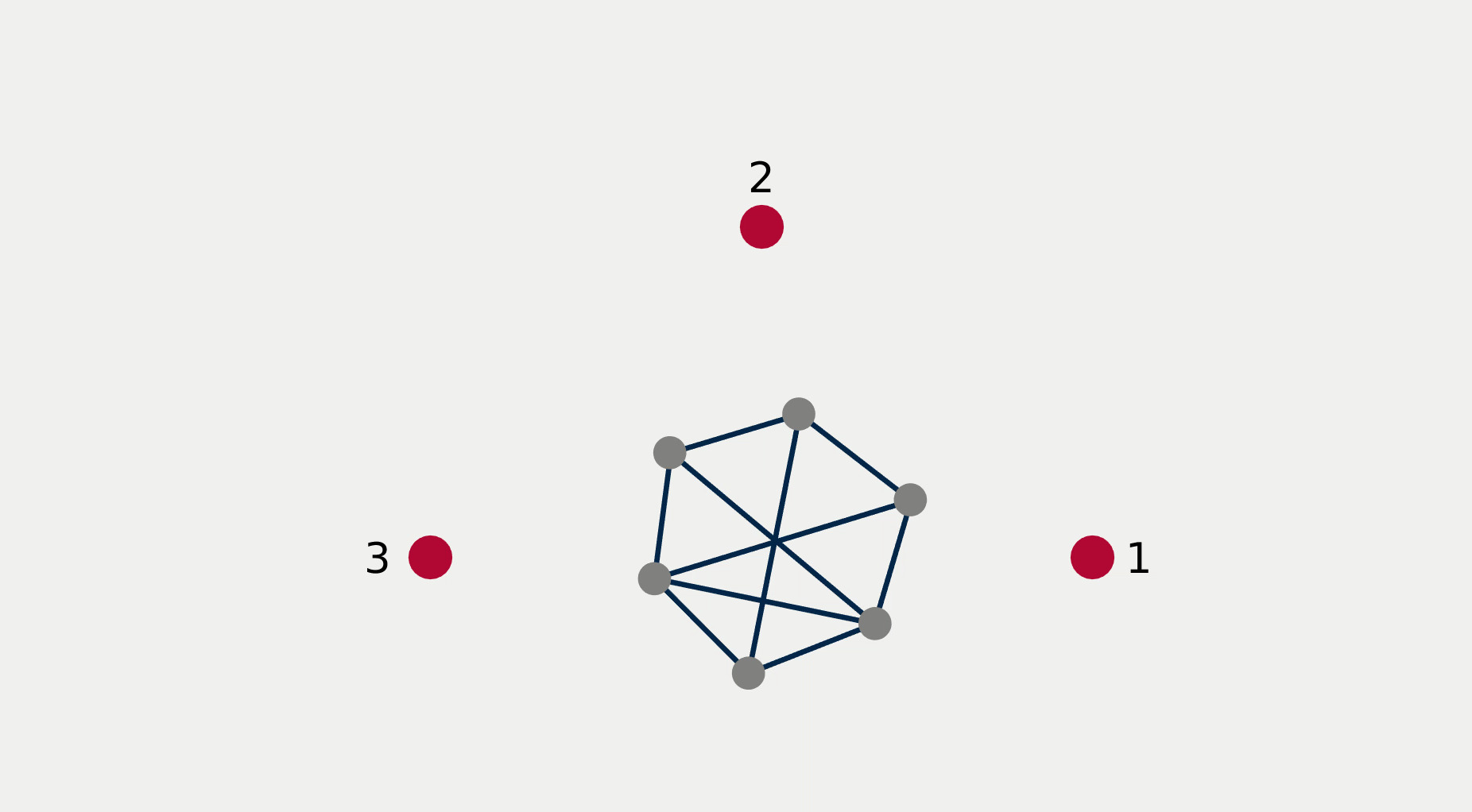}}\hfill
	\subfloat[][t = 38s]{\includegraphics[width=0.245\textwidth]{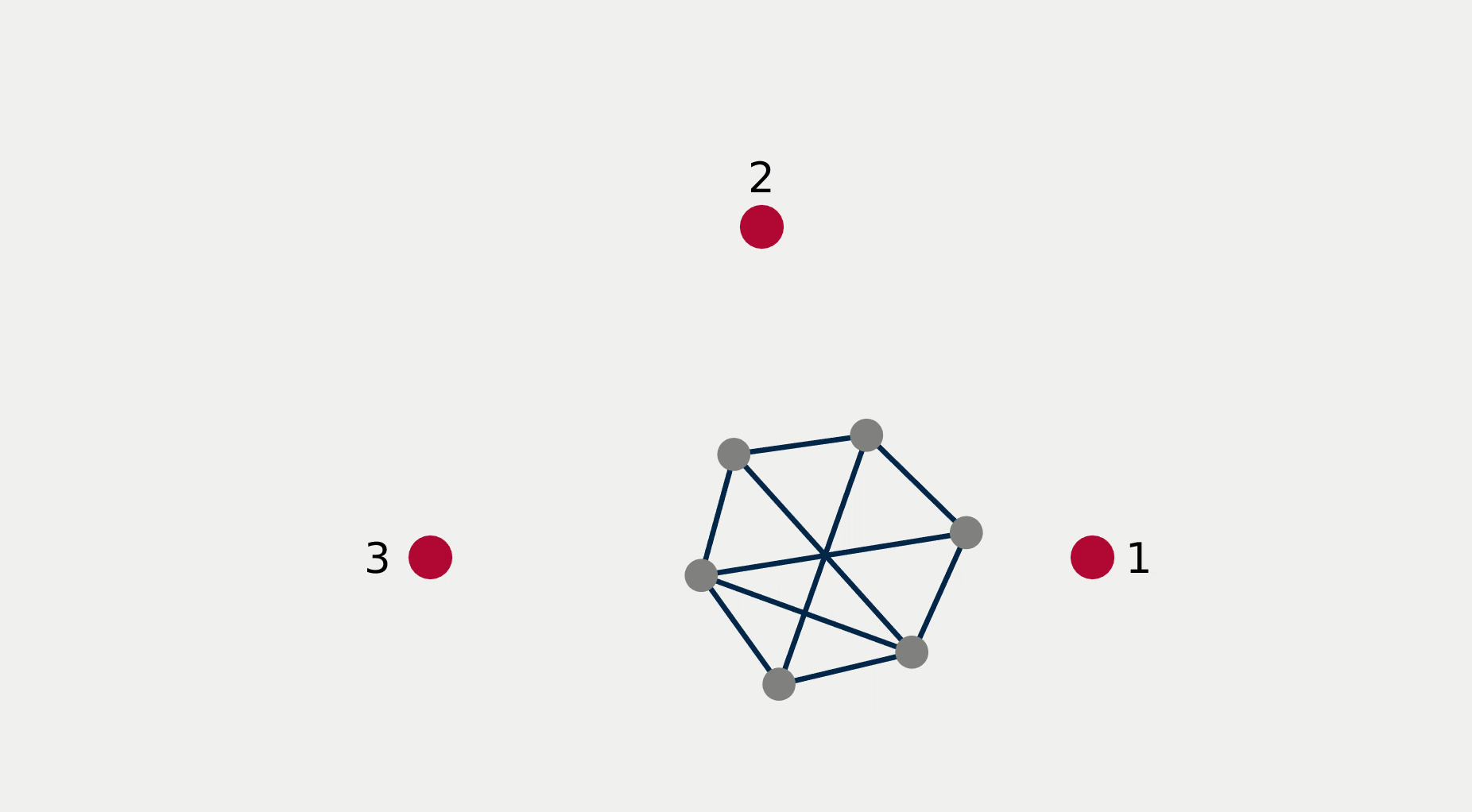}}\hfill
	\subfloat[][t = 42s]{\includegraphics[width=0.245\textwidth]{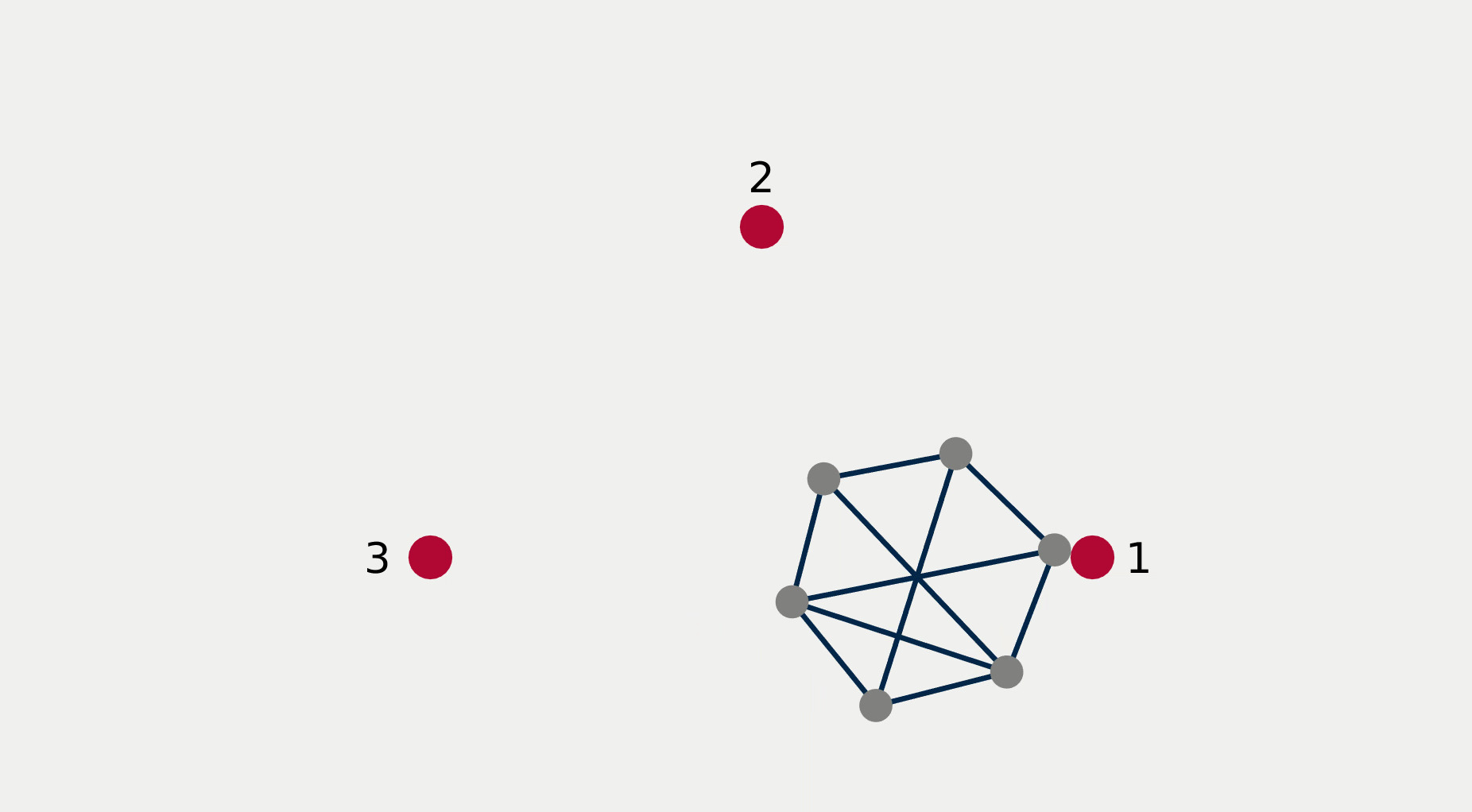}}\hfill
	\subfloat[][t = 45s]{\label{subfig:multi12}\includegraphics[width=0.245\textwidth]{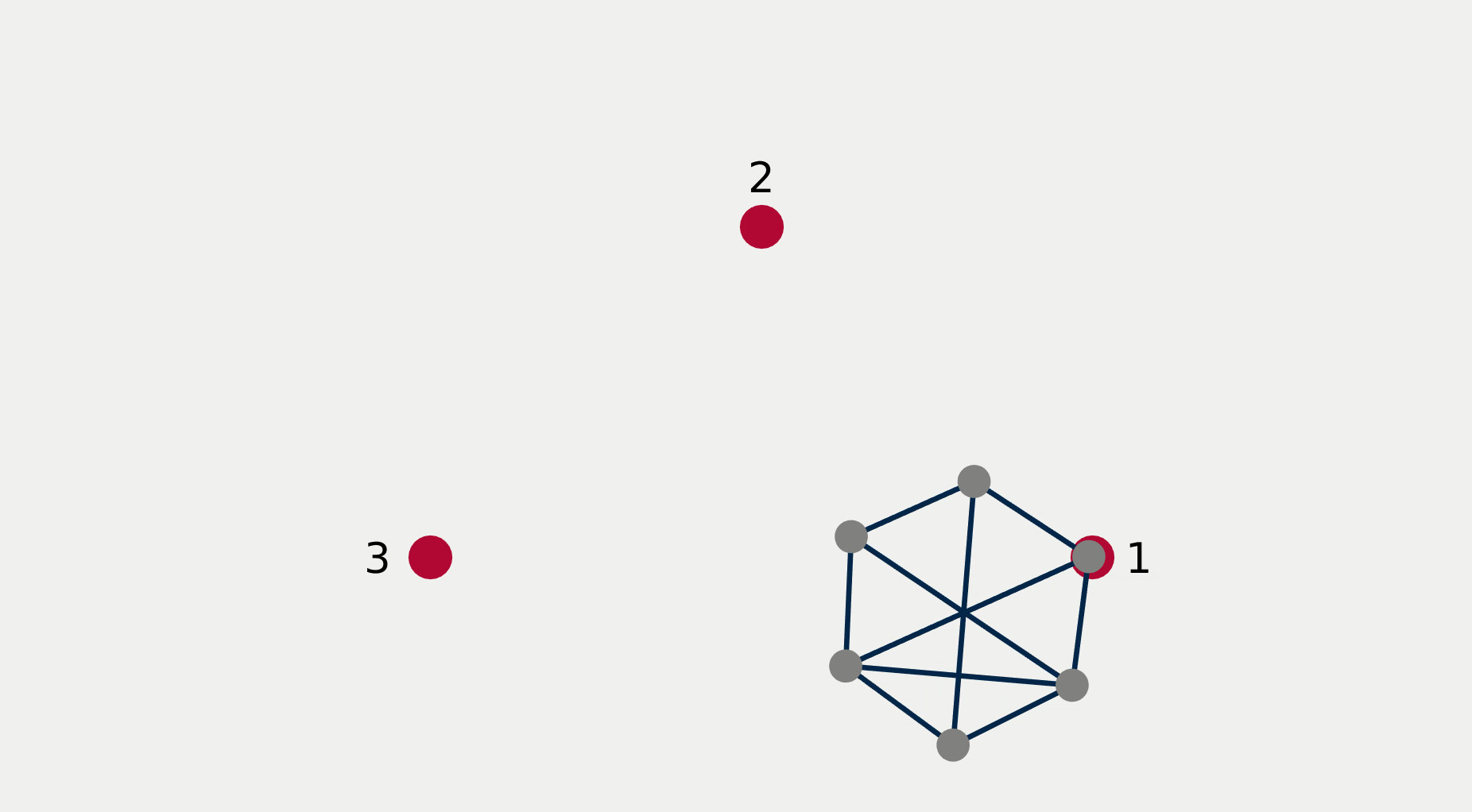}}\\
	\subfloat[][]{\label{subfig:multigraph}\includegraphics[trim={4cm 0cm 4cm 0cm}, clip,width=0.7\textwidth]{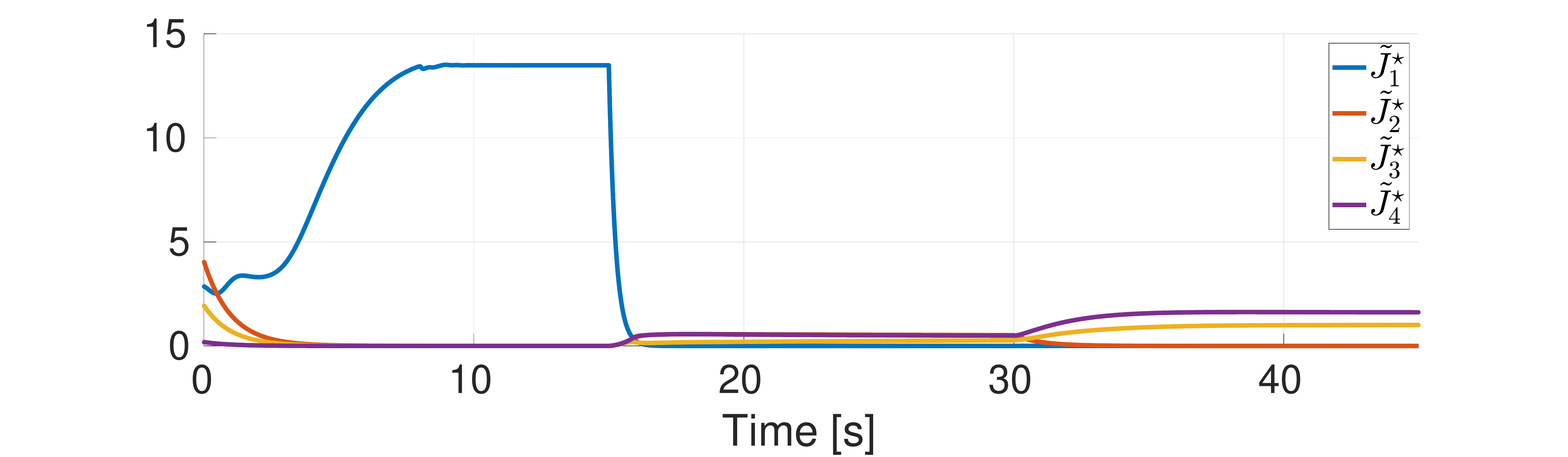}}
	\caption{Snapshots (\protect\ref{subfig:multi1}-\protect\ref{subfig:multi12}) and plot of $\tilde J^\star_1$, $\tilde J^\star_2$, $\tilde J^\star_3$, and $\tilde J^\star_4$ (\protect\ref{subfig:multigraph}) corresponding to the hexagonal formation control task and 3 go-to-goal tasks for robots 1, 2, and 3, respectively, recorded during the course of a simulated experiment with a multi-robot system. Robots are gray dots, connection edges between the robots used to assemble the desired formation are depicted in blue, goal points are shown as red dots.}
	\label{fig:Jstar_multirobot}
\end{figure*}

For multi-robot systems, the redundancy stems from the multiplicity of robotic units of which the system is comprised. In this section, we will showcase the execution of dependent tasks---two tasks are dependent if executing one prevents the execution of the other \cite{antonelli2009tro}---in different orders of priority. The multi-robot system is comprised of 6 planar robots modeled with single integrator dynamics and controlled to execute the following 4 tasks: All robots assemble an hexagonal formation (task $T_1$), robot 1 goes to goal point 1 (task $T_2$), robot 2 goes to goal point 2 (task $T_3$), robot 3 goes to goal point 3 (task $T_4$). While Task 1 is independent of each of the other tasks taken singularly, it is not independent of any pair of tasks 2, 3, and 4. This intuitively corresponds to the fact that it is possible to form a hexagonal formation in different points in space, but it might not be feasible to form a hexagonal formation while two robots are constrained to be in two pre-specified arbitrary locations.

Figure~\ref{fig:Jstar_multirobot} reports a sequence of snapshots and the graph of the value functions encoding the four tasks recorded during the course of the experiment. Denoting by $T_i\prec T_j$ the condition under which task $T_i$ has priority higher than $T_j$, the sequence of prioritized stacks tested in the experiment are the following:
\begin{equation}
	\label{eq:taskstacks}
	\begin{cases}
		T_2, T_3, T_4\prec T_1 & \quad 0s\le t< 15s\\
		T_1\prec T_2, T_3, T_4 & \quad 15s\le t< 30s\\
		T_1\prec T_2\prec T_3, T_4 & \quad 30s\le t\le 45s.
	\end{cases}
\end{equation}

The plot of the value functions in Fig.~\ref{subfig:multi12} shows how, for $0s\le t< 15s$, since the hexagonal formation control algorithm has lower priority compared to the three go-to-goal tasks, its value function $\tilde J^\star_1$ is allowed to grow while the other three value functions are driven to 0 by the velocity control input solution of \eqref{eq:multiRLtaskexecution} supplied to the robots. For $15s\le t< 30s$, the situation is reversed: the hexagonal formation control is executed with highest priority while the value functions encoding the three go-to-goal tasks are allowed to grow---a condition which corresponds to the non-execution of the tasks. Finally, for $30s\le t\le 45s$, task $T_2$, i.e. go-to-goal task for robot 1 to goal point 1 is added at higher priority with respect to tasks $T_3$ and $T_4$. Since this is independent by task $T_1$, it can be executed at the same time. As a result, as can be seen from the snapshots, the formation translates towards the red point marked with 1. Tasks $T_1$ and $T_2$ are successfully executed while tasks $T_3$ and $T_4$ are not executed since are not independent by the first two and they have lower priority.

\begin{remark}
	The optimization program responsible for the execution of multiple prioritized tasks encoded by value functions is solved at each iteration of the robot control loop. This illustrates how the convex optimization formulation of the developed framework is computationally efficient and therefore amenable to be employed in online settings. Alternative approaches for task prioritization and allocation in the context of multi-robot systems generally result in \mbox{(mixed-)integer} optimization programs, which are often characterized by a combinatorial nature and are not always suitable for an online implementation \cite{gerkey2004formal}.
\end{remark}

\subsection{Discussion}
\label{subsec:discussion}

The experiments of the previous section highlight several amenable properties of the framework developed in this paper for the prioritized execution of tasks encoded by a value function. First of all, its \textit{compositionality} is given by the fact that tasks can easily be inserted and removed by adding and removing constraints from the optimization program \eqref{eq:multiRLtaskexecution}. For the same reason the framework is \textit{incremental} and \textit{modular} as it allows for building a complex task using a number of subtasks which can be incrementally added to the constraints of an optimization-based controller. Moreover, it allows for seamless incorporation of priorities among tasks, and, as we showcased in Section~\ref{subsec:multirobot}, these priority can also be switched in an online fashion, in particular without the need of stopping and restarting the motion of the robots. Furthermore, Proposition~\ref{prop:stability} shows that the execution of multiple tasks using the constraint-driven control is \textit{stable} and the robotic system will indeed execute the given tasks according to the specified priorities. Finally, as the developed optimization program is a convex QP, its low \textit{computational complexity} allows for an efficient implementation in online settings even under real-time constraints on computationally limited robotic platforms.

\section{Conclusion}
\label{sec:conclusion}

In this paper, we presented an optimization-based framework for the prioritized execution of multiple tasks encoded by value functions. The approach combines control-theoretic and learning techniques in order to exhibit properties of compositionality, incrementality, stability, and low computational complexity. These properties render the proposed framework suitable for online and real-time robotic implementations. A multi-robot simulated scenario illustrated its effectiveness in the control of a redundant robotic system executing a prioritized stack of tasks.

\appendix

\section{Comparison Between Optimal Control, Optimization-Based Control, and RL policy}
\label{app:equivalence}

\begin{figure}
	\centering
	\includegraphics[width=0.95\textwidth,trim={4cm 0.5cm 4cm 1.5cm 0},clip]{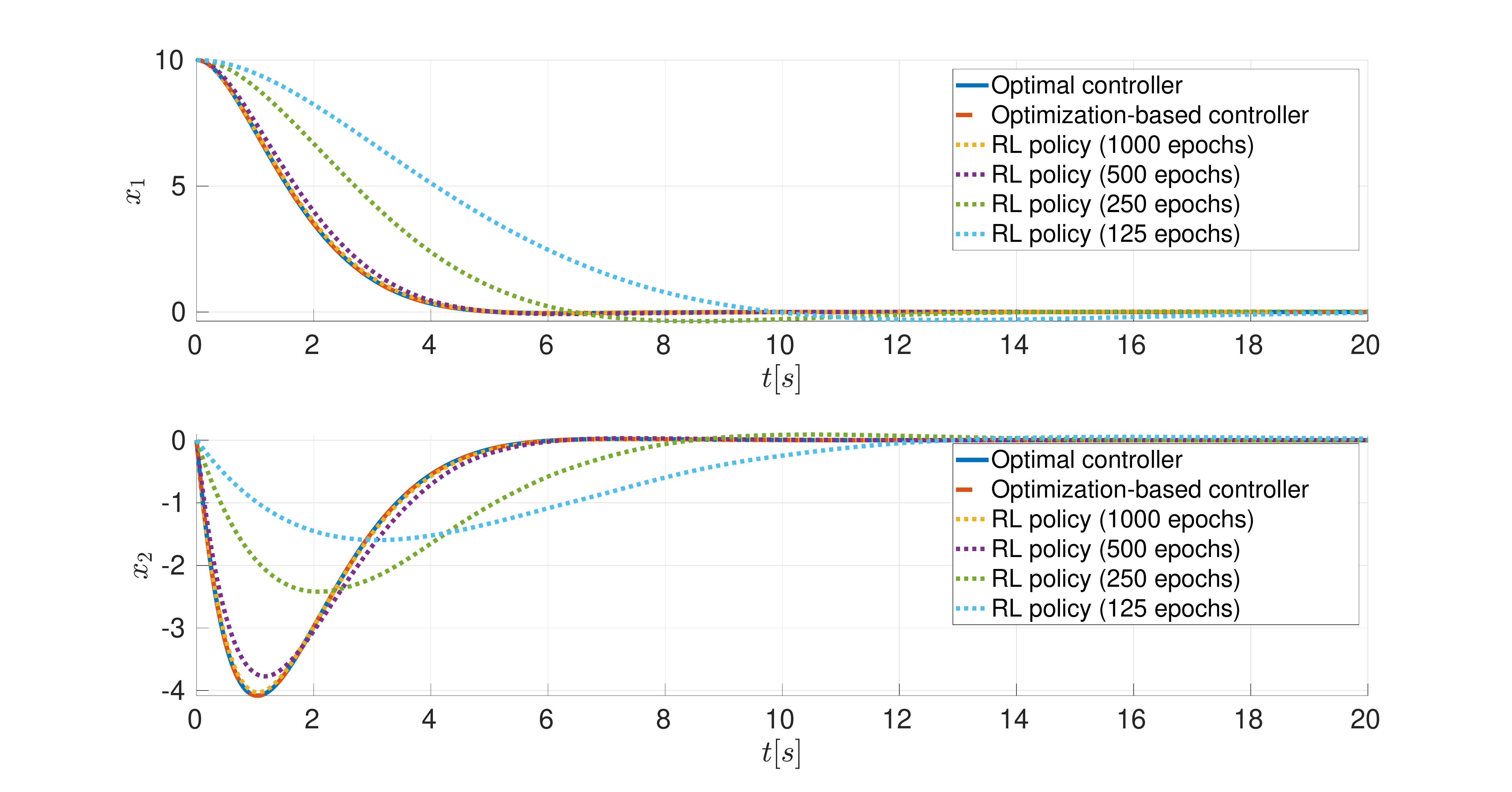}
	\caption{Comparison between the optimal controller (given in \eqref{eq:optimalpolicy}), RL policy (based on the approximate value function $\tilde J^\star$ \eqref{eq:optimalcosttogo}), and optimization-based controller (solution of \eqref{eq:minnorm} with $V=\tilde J^\star$) employed to stabilize a double-integrator system to the origin of its state space, i.e. driving both $x_1$ and $x_2$ to 0. As can be seen, when trained for a sufficiently long time, the RL policy results in the optimal controller, which is also equivalent to the optimization-based controller.}
	\label{fig:comparison}
\end{figure}

To compare optimal controller, optimization-based controller, and RL policy, in this section, we consider the stabilization of a double integrator system to the origin. The system dynamics are given by:
$\dot x = \begin{bmatrix}
	0&1\\
	0&0
\end{bmatrix} x + \begin{bmatrix}
	0\\
	1
\end{bmatrix} u$,
where $x=[x_1,x_2]\tr\in\R^2$ and $u\in\R$. The instantaneous cost considered in the optimal control problem \eqref{eq:optimalcontrolproblem} is given by $q(x) + u^2$ where $q(x) = x\tr x$. The reward function of the value iteration algorithm employed to learn an approximate representation of the value function has been set to $g(x,u) = -q(x) - u^2$, and the resulting value function $\tilde J^\star$ has been shifted so that $\tilde J^\star(0) = 0$.

The results of the comparison are reported in Fig.~\ref{fig:comparison}. Here, the optimization-based controller solution of \eqref{eq:minnorm} with $V=\tilde J^\star$ is compared to the optimal controller given in \eqref{eq:optimalpolicy}, and the RL policy corresponding to the approximate value function $\tilde J^\star$. As can be seen, the optimization-based controller and the optimal controller coincide, while the RL policy becomes closer and closer as the number of training epochs increases.

\section{Implementation Details}
\label{app:implementation}

The results reported in Section~\ref{sec:result} have been obtained using a custom value function learning algorithm written in Python. The details of each multi-robot task are given in the following.

Each robot in the team of $N$ robots is modeled using single integrator dynamics $\dot x_i = u_i$, where $x_i,u_i\in\R^2$ are position and velocity input of robot $i$. The ensemble state and input will be denoted by $x$ and $u$, respectively. For the formation control task, the expression of the cost $g$ is given by $g(x,u) = 1000 - 0.01 (-\mc E(x) - 10\|u\|^2)$, where the value of $\mc E(x)$ is the formation energy defined as $\mc E(x) = \sum_{i=1}^N \sum_{j\in\msc N_i} (\|x_i-x_j\|^2-W_{ij}^2)^2$,
$\msc N_i$ being the neighborhood of robot $i$, i.e. the set of robots with which robot $i$ shares an edge, and
\begin{equation}
	\label{eq:app:weightmatrix}
	W = \begin{bmatrix}
		0 & l & \sqrt{3} l & 2l & 0 & l\\
		l & 0 & l & 0 & 2l & 0\\
		\sqrt{3} l & l & 0 & l & 0 & 2l\\
		2l & 0 & l & 0 & l & 0\\
		0 & 2l & 0 & l & 0 & l\\
		l & 0 & 2l & 0 & l & 0
	\end{bmatrix}
\end{equation}
with $l=1$. The entry $ij$ of the matrix $W$ corresponds to the desired distance to be maintained between robots $i$ and $j$.

The cost function $g$ for the go-to-goal tasks is given by $g(x,u) = 100 - 0.01 (-\|x-\hat x\|^2 - \| u \|^2)$, where $\hat x\in\R^2$ is the desired goal point.

\begin{remark}[Combination of single-robot and multi-robot tasks]
Single-robot tasks (e.g. the go-to-goal tasks considered in this paper) are combined with multi-robot tasks (e.g. the formation control task) by defining the task gradient required to compute $L_{f_0} \tilde J_i^\star(x)$ and $L_{f_1} \tilde J_i^\star(x)$ in the optimization program \eqref{eq:multiRLtaskexecution} in the following way:
$\frac{\partial \tilde J_i^\star}{\partial x} = \begin{bmatrix} 0 & \cdots & 0 & \frac{\partial \tilde J_{ij}^\star}{\partial x} & 0 & \cdots & 0 \end{bmatrix}$,
where the $j$-th entry $\tilde J_{ij}^\star$ is the approximate value function for task $i$ and robot $j$.
\end{remark}

\bibliographystyle{spmpsci}
\bibliography{bib/dars2022refs}

\end{document}